%% file: main_repeated.tex
\documentclass[11pt]{article}
\usepackage{vitercik}

\usepackage{amsfonts}
\usepackage{bm}
\usepackage{thm-restate}
\usepackage{float} 
\usepackage{subcaption}

\def\reals{{\mathbb R}}

\usepackage{todonotes}

\def\pr{\mathbb{P}}
\newcommand{\ind}{\mathds{1}}
\def\calA{\mathcal{A}}
\def\cA{\mathcal{A}}
\def\cU{\mathcal{U}}

\def\calU{\mathcal{U}}
\def\reals{{\mathbb R}}

\def\R{\mathbb{R}}

\usepackage{dsfont}
\title{Learning to Prune: Speeding up Repeated Computations}
\date{\today}
\author{
Daniel Alabi \\ \small Harvard University \\ \small \texttt{alabid@g.harvard.edu}
\and 
Adam Tauman Kalai \\ \small Microsoft Research \\ \small \texttt{noreply@microsoft.com}
\and 
Katrina Ligett \\ \small Hebrew University \\ \small \texttt{katrina@cs.huji.ac.il}
\and 
Cameron Musco \\ \small Microsoft Research \\ \small \texttt{camusco@microsoft.com} 
\and 
Christos Tzamos \\ \small University of Wisconsin-Madison \\ \small \texttt{tzamos@wisc.edu}
\and 
Ellen Vitercik\\ \small Carnegie Mellon University \\ \small \texttt{vitercik@cs.cmu.edu}}

\begin{document}

\maketitle

\begin{abstract}
\input{abstract}
\end{abstract}

\section{Introduction}\label{sec:intro}
\input{intro}

\subsection{Related work}
\label{sec:rw}
\input{related}

\section{Model}\label{sec:model}
\input{model_repeated}
\input{model_pruning}

\section{The algorithm}\label{sec:expansive}
\input{expansive}

\section{Lower bound on the tradeoff between accuracy and runtime}\label{sec:lower}
\input{lowerbound}

\section{Experiments}\label{sec:experiments}
\input{experiments}

\section{Multiple solutions and approximations}\label{sec:multiple}
In this work, we have assumed that each problem has a unique solution, which we can enforce by  defining a canonical ordering on solutions. For string matching, this could be the first match in a string as opposed to any match. For shortest-path routing, it is not difficult to modify shortest-path algorithms to find, among the shortest paths, the one with lexicographically ``smallest'' description given some  ordering of edges. Alternatively, one might simply assume that there is exactly one solution, e.g., no ties in a shortest-path problem with real-valued edge weights. This latter solution is what we have chosen for the linear programming model, for simplicity.

It would be natural to try to extend our work to problems that have multiple solutions, or even to approximate solutions. 
However, addressing multiple solutions in repeated computation rapidly raises NP-hard challenges. To see this, consider a graph with two nodes, $s$ and $t$, connected by $m$ parallel edges. Suppose the goal is to find any shortest path and suppose that in each period, the edge weights are all 0 or 1, with at least one edge having weight 0. If $Z_i$ is the set of edges with 0 weight on period $i$, finding the smallest pruning which includes a shortest path on each period is trivially equivalent to set cover on the sets $Z_i$. Hence, any repeated algorithm handling problems with multiple solutions must address this computational hardness.

\section{Conclusion}
\input{conclusion}

\section*{Acknowledgments}
This work was supported in part by Israel Science Foundation (ISF) grant \#1044/16, a subcontract on the DARPA Brandeis Project, and the Federmann Cyber Security Center in conjunction with the Israel national cyber directorate.

\bibliography{online}
\bibliographystyle{plainnat}
\appendix

\section{Proofs from Section~\ref{sec:expansive}}\label{app:expansive}
\input{appendix_expansive}

\section{Proof of lower bound}\label{ap:lowerbound}
\input{appendix_lowerbound}

\section{Additional information about experiments}\label{app:experiments}
\input{appendix_experiments.tex}

\end{document}

%% file: abstract.tex
It is common to encounter situations where one must solve a sequence of similar computational problems. Running a standard algorithm with worst-case runtime guarantees on each instance will fail to take advantage of valuable structure shared across the problem instances. For example, when a commuter drives from work to home, there are typically only a handful of routes that will ever be the shortest path. A na\"ive algorithm that does not exploit this common structure may spend most of its time checking roads that will never be in the shortest path. More generally, we can often ignore large swaths of the search space that will likely never contain an optimal solution. 

We present an algorithm that learns to maximally prune the search space on repeated computations, thereby reducing runtime while provably outputting the correct solution each period with high probability. Our algorithm employs a simple explore-exploit technique resembling those used in online algorithms, though our setting is quite different. We prove that, with respect to our model of pruning search spaces, our approach is optimal up to constant factors. Finally, we illustrate the applicability of our model and algorithm to three classic problems: shortest-path routing, string search, and linear programming. We present experiments confirming that our simple algorithm is effective at significantly reducing the runtime of solving repeated computations.

%% file: intro.tex
Consider computing the shortest path from home to work every morning. The shortest path may vary from day to day---sometimes side roads beat the highway; sometimes the bridge is closed due to construction. However, although San Francisco and New York are contained in the same road network, it is unlikely that a San Francisco-area commuter would ever find New York along her shortest path---the edge times in the graph do not change \emph{that} dramatically from day to day.

With this motivation in mind, we study a learning problem where the goal is to speed up repeated computations when the sequence of instances share common substructure. Examples include repeatedly computing the shortest path between the same two nodes on a graph with varying edge weights, repeatedly computing string matchings,
and repeatedly solving linear programs with mildly varying objectives. Our work is in the spirit of recent work in data-driven algorithm selection \citep[e.g.,][]{GuptaR17,Balcan17:Learning,Balcan18:Learning,Balcan18:Dispersion} and online learning  \citep[e.g.,][although with some key differences, which we discuss below]{cesa2006prediction}.

The basis of this work is the observation that for many realistic instances of repeated problems, vast swaths of the search space may never contain an optimal solution---perhaps the shortest path is always contained in a specific region of the road network; large portions of a DNA string may never contain the patterns of interest; a few key linear programming constraints may be the only ones that bind.
Algorithms designed to satisfy worst-case guarantees may thus waste substantial computation time on futile searching. For example, even if a single, fixed path from home to work were best every day, Dijkstra's algorithm would consider all nodes within distance $d_i$ from home on day $i$, where $d_i$ is the length of the optimal path on day $i$, as illustrated in Figure~\ref{fig:city}.
\begin{figure}
\begin{center}
\includegraphics[width=\textwidth]{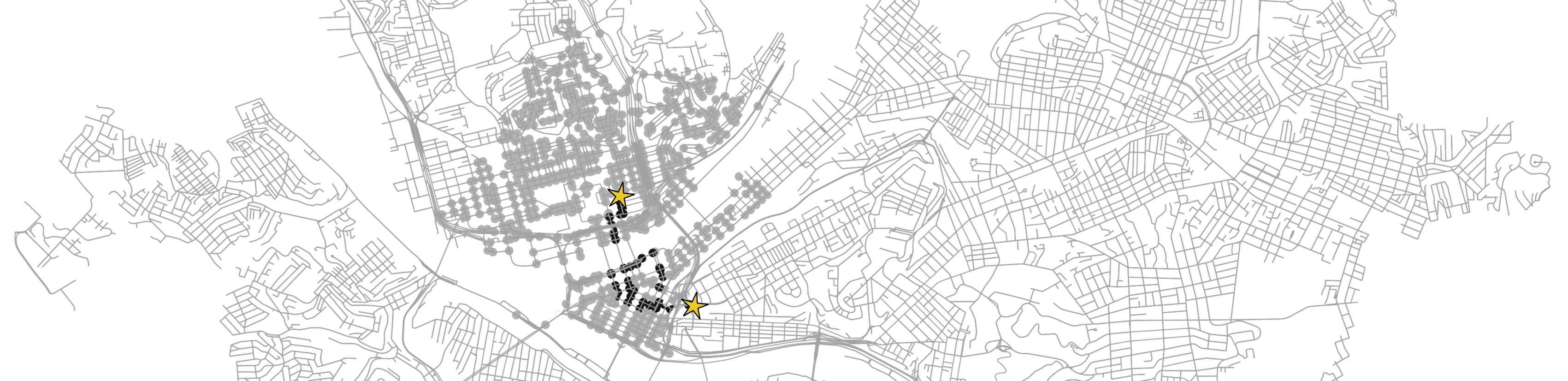}
\end{center}
\caption{A standard algorithm computing the shortest path from the upper to the lower star will explore many nodes (grey), even nodes in the opposite direction. Our algorithm learns to prune to a subgraph (black) of nodes that have been included in prior shortest paths.\label{fig:city}}
\end{figure}

We develop a simple solution, inspired by online learning, that leverages this observation to the maximal extent possible. On each problem, our algorithm typically searches over a small, pruned subset of the solution space, which it learns over time. This pruning is the minimal subset containing all previously returned solutions. These rounds are analogous to ``exploit'' rounds in online learning. To learn a good subset, our algorithm occasionally deploys a worst-case-style algorithm, which explores a large part of the solution space and guarantees correctness on any instance. These rounds are analogous to ``explore'' rounds in online learning. If, for example, a single fixed path were always optimal, our algorithm would almost always immediately output that path, as it would be the only one in its pruned search space. Occasionally, it would run a full Dijkstra's computation to check if it should expand the pruned set. Roughly speaking, we prove that our algorithm's solution is almost always correct, but its cumulative runtime is not much larger than that of running an optimal algorithm on the maximally-pruned search space in hindsight. Our results hold for worst-case sequences of problem instances, and we do not make any distributional assumptions.

In a bit more detail, let $f:X \rightarrow Y$ be a function that takes as input a problem instance $x \in X$ and returns a solution $y \in Y$. Our algorithm receives a sequence of inputs from $X$. Our high-level goal is to correctly compute $f$ on almost every round while minimizing runtime. For example, each $x \in X$ might be a set of graph edge weights for some fixed graph $G = (V,E)$ and $f(x)$ might be the shortest $s$-$t$ path for some vertices $s$ and $t$. Given a sequence $x_1, \dots, x_T \in X$, a worst-case algorithm would simply compute and return $f(x_i)$ for every instance $x_i$. However, in many application domains, we have access to other functions mapping $X$ to $Y$, which are faster to compute. These simpler functions are defined by subsets $S$ of a universe $\calU$ that represents the entire search space. 
We call each subset a ``pruning'' of the search space. For example, in the shortest paths problem, $\calU$ equals the set $E$ of edges and a pruning $S \subset E$ is a subset  of the edges. The function corresponding to $S$, which we denote $f_S: X \to Y$, also takes as input edge weights $x$, but returns the shortest path from $s$ to $t$ using only edges from the set $S$. By definition, the function that is correct on every input is $f = f_{\calU}$. 
We assume that for every $x$, there is a set $S^*(x) \subseteq \mathcal{U}$ such that $f_S(x)=f(x)$ if and only if $S \supseteq S^*(x)$ -- a mild assumption we discuss in more detail later on.

Given a sequence of inputs $x_1, \dots, x_T$, our algorithm returns the value $f_{S_i}(x_i)$ on round $i$, where $S_i$ is chosen based on the first $i$ inputs $x_1, \dots, x_i$. Our goal is two fold: first, we hope to minimize the size of each $S_i$ (and thereby maximally prune the search space), since $|S_i|$ is often monotonically related to the runtime of computing $f_{S_i}(x_i)$. For example, a shortest path computation will typically run faster if we consider only paths that use  a small subset of edges. To this end, we prove that if $S^*$ is the smallest set such that $f_{S^*}(x_i) = f(x_i)$ for all $i$ (or equivalently,\footnote{We explain this equivalence in Lemma~\ref{lem:Sstar}.} $S^* = \bigcup_{i = 1}^T S^*(x_i)$), then \[\E\left[\frac{1}{T}\sum_{i = 1}^T|S_i|\right] \leq |S^*| + \frac{|\calU| - |S^*|}{\sqrt{T}},\] where the expectation is over the algorithm's randomness. At the same time, we seek to minimize the the number of mistakes the our algorithm makes (i.e., rounds $i$ where $f(x_i) \not= f_{S_i}(x_i)$). We prove that the expected fraction of rounds $i$ where $f_{S_i}(x_i) \not= f(x_i)$ is
$O(|S^*|/\sqrt{T})$. Finally, the expected runtime\footnote{As we will formalize, when determining $S_1,\dots,S_T$, our algorithm must compute the smallest set $S$ such that $f_S(x_i) = f(x_i)$ on some of the inputs $x_i$. In all of the applications we discuss, the total runtime required for these computations is upper bounded by the total time required to compute $f_{S_i}(x_i)$ for $i \in [T].$} of the algorithm is the expected time required to compute $f_{S_i}(x_i)$ for $i \in [T]$, plus $O(|S^*| \sqrt{T})$ expected time to determine the subsets $S_1, \dots, S_T$.

We instantiate our algorithm and corresponding theorem in three diverse settings---shortest-path routing, linear programming, and string matching---to illustrate the flexibility of our approach. We present experiments on real-world maps and economically-motivated linear programs. In the case of shortest-path routing, our algorithm's performance is illustrated in Figure~\ref{fig:city}. Our algorithm explores up to five times fewer nodes on average than Dijkstra's algorithm, while sacrificing accuracy on only a small number of rounds. In the case of linear programming, when the objective function is perturbed on each round but the constraints remain invariant, we show that it is possible to significantly prune the constraint matrix, allowing our algorithm to make fewer simplex iterations to find solutions that are nearly always optimal.

%% file: related.tex
Our work advances a recent line of research studying the foundations of algorithm configuration. Many of these works study a distributional setting~\citep{AilonCCLMS11,ClarksonMS14,GuptaR17,Kleinberg17:Efficiency,Balcan17:Learning,Balcan18:Learning,Balcan18:Dispersion,Weisz18:LEAPSANDBOUNDS}: there is a distribution over problem instances and the goal is to use a set of samples from this distribution to determine an algorithm from some fixed class with the best expected performance. In our setting, there is no distribution over instances: they may be adversarially selected.
Additionally, we focus on quickly computing solutions for polynomial-time-tractable problems rather than on developing algorithms for NP-hard problems, which have been the main focus of prior work.

Several works have also studied online algorithm configuration without distributional assumptions from a theoretical perspective~\citep{GuptaR17,Cohen-Addad17:Online,Balcan18:Dispersion}. Before the arrival of any problem instance, the learning algorithm fixes a class of algorithms to learn over. The classes  of algorithms that \citet{GuptaR17}, and \citet{Cohen-Addad17:Online}, and \citet{Balcan18:Dispersion} study are infinite, defined by real-valued parameters. The goal is to select parameters at each timestep while minimizing regret. These works provide conditions under which it is possible to design algorithms achieving sublinear regret. These are conditions on the cost functions mapping the real-valued parameters to the algorithm's performance on any  input. In our setting, the choice of a pruning $S$ can be viewed as a parameter, but this parameter is combinatorial, not real-valued, so the prior analyses do not apply.

Several works have studied how to take advantage of structure shared over a sequence of repeated computations for specific applications, including linear programming~\citep{Banerjee15efficientlysolving} and matching~\citep{Deb06fastmatching}. As in our work, these algorithms have full access to the problem instances they are attempting to solve. These approaches are quite different (e.g., using machine classifiers) and highly tailored to the application domain, whereas we provide a general algorithmic framework and instantiate it in several different settings.

Since our algorithm receives input instances in an online fashion and makes no distributional assumptions on these instances, our setting is reminiscent of online optimization. However, unlike the typical online setting, we observe each input $x_i$ \emph{before} choosing an output $y_i$. Thus, if runtime costs were not a concern, we could always return the best output for each input. We seek to trade off correctness for lower runtime costs.
In contrast, in online optimization, one must commit to an output $y_i$ before seeing each input $x_i$, in both the full information and bandit settings \citep[see, e.g.,][]{KalaiV05,AwerbuchK08}. In such a setting, one cannot hope to return the best $y_i$ for each $x_i$ with significant probability. Instead, the typical goal is that the performance over all inputs should compete with the performance of the best fixed output in hindsight.

%% file: model_repeated.tex
We start by defining our model of repeated computation.
Let $X$ be an abstract set of problem instances and let $Y$ be a set of possible solutions. We design an algorithm that operates over $T$ rounds: on round $i$, it receives an instance $x_i \in X$ and returns some element of $Y$. 

\begin{definition}[Repeated algorithm]
Over $T$ rounds, a repeated algorithm $\mathcal{A}$ encounters a sequence of inputs
$x_1, x_2, \ldots,x_T \in X$. On round $i$, after receiving input $x_i$, it outputs
$\calA(x_{1:i})\in Y$, where $x_{1:i}$ denotes the sequence $x_1, \dots, x_i$. A repeated algorithm may maintain a state from period to period, and thus $\mathcal{A}(x_{1:i})$ may potentially depend on all of $x_1,...,x_i$.
\end{definition}

We assume each problem instance $x \in X$ has a unique correct solution (invoking tie-breaking assumptions as necessary; in Section \ref{sec:multiple}, we discuss how to handle problems that admit multiple solutions). We denote the mapping from instances to correct solutions as $f:X \rightarrow Y$.
For example, in the case of shortest paths, we fix a graph $G$ and a  pair $(s,t)$ of source and terminal nodes. Each instance $x \in X$ represents a weighting of the graph's edges. The set $Y$ consists of all paths from $s$ to $t$ in $G$. Then
$f(x)$ returns the shortest path from $s$ to $t$ in $G$, given the edge weights $x$ (breaking ties according to some canonical ordering of the elements of $Y$, as discussed in Section \ref{sec:multiple}).
To measure correctness, we use a \textit{mistake bound model} \citep[see, e.g.,][]{Littlestone87}.

\begin{definition}[Repeated algorithm mistake bound]
  The mistake bound
  of the repeated algorithm $\calA$ given inputs $x_1, \dots, x_T$ is
  \[M_T(\calA, x_{1:T}) = \E\left[\sum_{i=1}^T\ind[\calA(x_{1:i}) \neq f(x_i)]\right],\] where the expectation is over the algorithm's random choices.
\end{definition}

%% file: model_pruning.tex
To minimize the number of mistakes, the na\"ive algorithm would simply compute the function $f(x_i)$ at every round $i$. However, in our applications, we will have the option of computing other functions mapping the set $X$ of inputs to the set $Y$ of outputs that are faster to compute than $f$. Broadly speaking, these simpler functions are defined by subsets $S$ of a universe $\calU$, or ``prunings'' of $\calU$. For example, in the shortest paths problem, given a fixed graph $G = (V,E)$ as well as source and terminal nodes $s, t \in V$, the universe is the set of edges, i.e., $\calU = E$. Each input $x$ is a set of edge weights and $f(x)$ computes the shortest $s$-$t$ path in $G$ under the input weights. The simpler function corresponding to a subset $S \subset E$ of edges also takes as input weights $x$, but it returns the shortest path from $s$ to $t$ using only edges from the set $S$ (with $f_S(x) = \bot$ if no such path exists). Intuitively, the universe $\calU$ contains all the information necessary to compute the correct solution $f(x)$ to any input $x$, whereas the function corresponding to a subset $S \subset \calU$ can only compute a subproblem using information restricted to $S$.

Let $f_S: X \to Y$ denote the function corresponding to the set $S \subseteq \calU$. We make two natural assumptions on these functions. First, we assume the function corresponding to the universe $\calU$ is always correct. Second, we assume there is a unique smallest set $S^*(x) \subseteq \calU$ that any pruning must contain in order to correctly compute $f(x)$. These assumptions are summarized below.

\begin{assumption}\label{assumption}
For all $x \in X$, $f_{\calU}(x) = f(x)$. Also, there exists a unique smallest set $S^*(x) \subseteq \calU$ such that $f_{\calU}(x) = f_S(x)$ if and only if $S^*(x) \subseteq S.$
\end{assumption}

Given a sequence of inputs $x_1, \dots, x_T$, our algorithm returns the value $f_{S_i}(x_i)$ on round $i$, where the choice of $S_i$ depends on the first $i$ inputs $x_1, \dots, x_i$.
In our applications, it is typically faster to compute $f_S$ over $f_{S'}$ if $|S| < |S'|$.
Thus, our goal is to minimize the number of mistakes the algorithm makes while simultaneously minimizing $\E\left[\sum|S_i|\right]$. Though we are agnostic to the specific runtime of computing each function $f_{S_i}$, minimizing $\E\left[\sum|S_i|\right]$ roughly amounts to minimizing the search space size and our algorithm's runtime in the applications we consider.

We now describe how this model can be instantiated in three classic settings: shortest-path routing, string search, and linear programming.

\paragraph{Shortest-path routing.}
In the repeated shortest paths problem, 
we are given a graph $G=(V, E)$ (with static structure) and a fixed pair $s,t \in V$ of source and terminal nodes. In period
$i\in[T]$, the algorithm receives a nonnegative weight assignment $x_i:E\rightarrow\reals_{\geq 0}$. Figure~\ref{fig:whole} illustrates
the pruning model applied to
the repeated shortest paths problem. 

\begin{figure}
\centering
\begin{subfigure}[t]{.47\textwidth}
\includegraphics[scale=1]{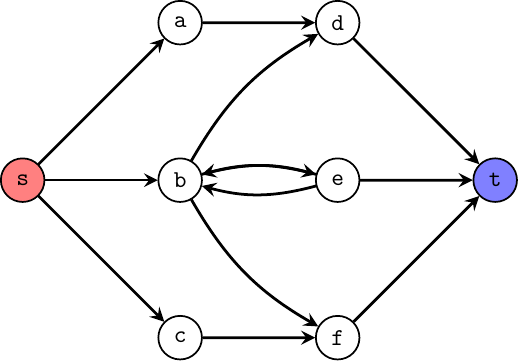}\centering
\caption{Original unweighted graph $G$.}\label{fig:a}
\end{subfigure}\qquad
\begin{subfigure}[t]{.47\textwidth}
\includegraphics[scale=1]{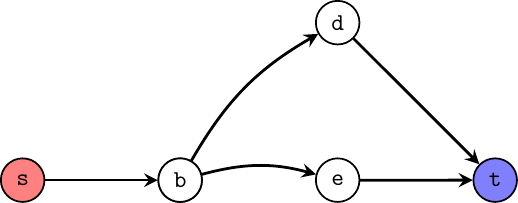}\centering
\caption{Optimal pruning.}\label{fig:b}
\end{subfigure}
\caption{Repeated shortest paths and optimal pruning of a graph $G$. If the shortest path was always $s$-$b$-$e$-$t$ or $s$-$b$-$d$-$t$, it would be unnecessary to search the entire graph for each instance.}
\label{fig:whole}
\end{figure}

For this problem, the universe is the edge set (i.e., $\calU = E$) and $S$ is a subset of edges in
the graph. The set $X$ consists of all possible weight assignments to edges in the graph $G$ and $Y\subseteq2^E \cup \left\{\bot\right\}$ is the
set of all paths in the graph, with $\bot$ indicating that no path exists. The function $f(x)$ returns the shortest $s$-$t$ path in $G$ given edge weights $x$.
For any $S\subseteq\calU$, the function $f_S: X\rightarrow Y$ computes the shortest $s$-$t$ path on the subgraph induced by the edges in $S$ (breaking ties by a canonical edge ordering). If $S$ does not include any $s$-$t$ path, we define $f_S(x) = \bot$.
Part 1 of Assumption~\ref{assumption} holds because $\calU = E$, so $f_{\calU}$ computes the shortest path on the entire graph. Part 2 of Assumption~\ref{assumption} also holds: since $f_S: X\rightarrow Y$ computes the shortest $s$-$t$ path on the subgraph induced by the edges in $S$ (breaking ties by some canonical edge ordering), we can see that  $f_S(x_i) = S^*(x_i)$ if and only if $S^*(x_i) \subseteq S$. To ``canonicalize'' the algorithm so there is always a unique solution, we assume there is a given ordering on edges and that ties are broken lexicographically according to the path description. This is easily achieved by keeping the heap maintained by Dijkstra's algorithm sorted not only by distances but also lexicographically.

\paragraph{Linear programming.}
We consider computing $\argmax_{\vec{y}  \in \R^n} \left\{\vec{x}^T\vec{y}: A\vec{y}\leq \vec{b}\right\}$, where we assume that
$(A, \vec{b})\in \reals^{m\times n} \times \reals^m$ is fixed across all times steps but the vector $\vec{x}_i \in X \subseteq \reals^n$ defining the objective function $\vec{x}_i^T\vec{y}$ may differ for each $i \in [T]$.
To instantiate our pruning model,
the universe $\calU = [m]$ is the set of all constraint indices and each $S \subseteq \calU$ indicates a subset of those constraints.
The set $Y$ equals $\reals^n \cup\{\bot\}$.
For simplicity, we assume that the set $X \subseteq \reals^n$ of objectives contains only directions $\vec{x}$ such that there is a unique solution $\vec{y}\in \reals^n$ that is the intersection of exactly $n$ constraints in $A$. This avoids both dealing with solutions that are the intersection of more than $n$ constraints and directions that are under-determined and have infinitely-many solutions forming a facet. See Section \ref{sec:multiple} for a discussion of this issue in general.

Given $\vec{x} \in \reals^n$, the function $f$ computes the linear program's optimal solution, i.e.,
$f(\vec{x}) = \argmax_{\vec{y}\in\reals^n} \left\{\vec{x}^T\vec{y}
 : A\vec{y}\leq \vec{b}\right\}.$
For a subset of constraints $S\subseteq\calU$, the function
$f_S$ computes the optimal solution restricted to those constraints, i.e.,
$f_S(\vec{x}) = \argmax_{\vec{y}\in\reals^n}\{\vec{x}^T\vec{y}
: A_S\vec{y}\leq \vec{b}_S\},$
where $A_S\in\reals^{|S|\times n}$ is the submatrix of $A$ consisting of the rows indexed by elements of $S$ and $\vec{b}_S\in\reals^{|S|}$ is the vector $\vec{b}$ with indices restricted to elements of $S$. We further write $f_S(\vec{x})=\bot$ if there is no unique solution to the linear program (which may happen for small sets $S$ even if the whole LP does have a unique solution).
Part 1 of Assumption~\ref{assumption} holds because $A_\calU = A$ and $\vec{b}_{\calU} = \vec{b}$, so it is indeed the case that $f_{\calU} = f$. To see why part 2 of Assumption~\ref{assumption} also holds, suppose that $f_S(\vec{x}) = f(\vec{x})$. If $f(\vec{x}) \not= \bot$, the vector $f_S(\vec{x})$ must be the intersection of exactly $n$ constraints in $A_S$, which by definition are indexed by elements of $S$. This means that $S^*(\vec{x}) \subseteq S$.

\paragraph{String search.}
In string search, the goal is to find the location of a short pattern in a long string. At timestep $i$, the algorithm receives a long string $q_i$ of some fixed length $n$ and a pattern $p_i$ of some fixed length $m \leq n$. We denote the long string as $q_i = \left(q_i^{(1)}, \dots, q_i^{(n)}\right)$ and the pattern as $p_i = \left(p_i^{(1)}, \dots, p_i^{(m)}\right)$.
The goal is to find an index $j\in[n-m+1]$ such that
$p_i = \left(q_i^{(j)}, q_i^{(j+1)}, \dots, q_i^{(j+m-1)}\right)$. The function $f$ returns the smallest such index $j$, or $\bot$ if there is no match.
In this setting, the set $X$ of inputs consists of all string pairs of length $n$ and $m$
(e.g., $\{A, T, G, C\}^{n \times m}$ for DNA sequences) and
the set $Y = [n-m+1]$ is the set of all possible match indices.
The universe $\calU = [n-m+1]$ also consists of all possible match indices.
For any $S \subseteq \calU$, the function $f_S(q_i, p_i)$ returns the smallest 
index $j \in S$ such that $p_i = \left(q_i^{(j)}, q_i^{(j+1)}, \dots, q_i^{(j+m-1)}\right)$, which we denote $j_i^*.$ It returns $\bot$ if there is no match.
We can see that part 1 of Assumption~\ref{assumption} holds: $f_\calU(q, p) = f(q, p)$ for all $(q, p) \in X$, since $f_{\calU}$ checks every index in $[n-m+1]$ for a match. Moreover, part 2 of Assumption~\ref{assumption} holds because $f_{\calU}(x_i) = f_S(x_i)$ if and only if $S^*(x_i) = \left\{j^*_i\right\} \subseteq S$.

%% file: expansive.tex
We now present an algorithm
(Algorithm~\ref{alg:0}), denoted $\cA^*$, that encounters a sequence of inputs $x_1, \dots, x_T$ one-by-one. At timestep $i$, it computes the value $f_{S_i}(x_i)$, where the choice of $S_i \subseteq \calU$ depends on the first $i$ inputs $x_1, \dots, x_i$.
We prove that, in expectation, the number of mistakes it makes (i.e., rounds where $f_{S_i}(x_i) \not= f(x_i)$) is small, as is $\sum_{i = 1}^T |S_i|$.

Our algorithm keeps track of a pruning of $\calU$, which we call $\bar{S}_i$ at timestep $i$. In the first round, the pruned set is empty ($\bar{S}_1 = \emptyset$). On round $i$, with some probability $p_i$, the algorithm computes the function $f_{\calU}(x_i)$ and then computes $S^*(x_i)$, the unique smallest set that any pruning must contain in order to correctly compute $f_{\calU}(x_i)$. (As we discuss in Section~\ref{sec:instantiations}, in all of the applications we consider, computing $S^*(x_i)$ amounts to evaluating $f_{\calU}(x_i)$.) The algorithm unions $S^*(x_i)$ with $\bar{S}_i$ to create the set $\bar{S}_{i+1}$. Otherwise, with probability $1-p_i$, it outputs $f_{\bar{S}_i}(x_i)$, and does not update the set $\bar{S}_i$ (i.e., $\bar{S}_{i+1} = \bar{S}_i$). It repeats in this fashion for all $T$ rounds.

\begin{algorithm}
\caption{Our repeated algorithm $\cA^*$}
\label{alg:0}
\begin{algorithmic}[1]
\State $\bar{S}_1 \leftarrow \emptyset$
\For{$i \in \{1, \dots, T\}$}
\State Receive input $x_i\in X$.
\State With probability $p_i$, output $f_{\calU}(x_i)$. Compute $S^*(x_i)$ and set $\bar{S}_{i+1} \leftarrow \bar{S}_{i} \cup S^*(x_i)$.
\State Otherwise (with probability $1-p_i$), output $f_{\bar{S}_i}(x_i)$ and set $\bar{S}_{i+1} \leftarrow \bar{S}_{i}$.
\EndFor
\end{algorithmic}
\end{algorithm}

In the remainder of this section, we use the notation $S^*$ to denote the smallest set such that $f_{S^*}(x_i) = f(x_i)$ for all $i \in [T]$. To prove our guarantees, we use the following helpful lemma:

\begin{lemma}\label{lem:Sstar}
For any $x_1, \dots, x_T \in X$, $S^* = \bigcup_{i = 1}^T S^*(x_i).$
\end{lemma}

\begin{proof}
First, we prove that $S^* \supseteq \cup_{i = 1}^TS^*(x_i)$. For a contradiction, suppose that for some $i \in [T]$, there exists an element $j \in S^*(x_i)$ such that $j \not \in S^*$. This means that $f_{S^*}(x_i)=f(x_i)$, but $S^*\not\supseteq S^*(x_i)$, which contradicts Assumption~\ref{assumption}: $S^*(x_i)$ is the unique smallest subset of $\calU$ such that for any set $S \subseteq \cU$, $f(x_i) = f_S(x_i)$ if and only if $S^*(x_i) \subseteq S.$ Therefore, $S^* \supseteq \cup_{i = 1}^TS^*(x_i)$. Next, let $C = \cup_{i = 1}^TS^*(x_i)$. Since $S^*(x_i) \subseteq C$, Assumption~\ref{assumption} implies that $f(x_i) = f_C(x_i)$ for all $i \in [T].$ Based on the definition of $S^*$ and the fact that $S^* \supseteq C$, we conclude that $S^* = C = \cup_{i = 1}^TS^*(x_i)$.
\end{proof}

We now provide a mistake bound for Algorithm~\ref{alg:0}.

\begin{theorem}
For any $p\in(0, 1]$ such that $p_i \geq p$ for all $i\in[T]$ and any inputs $x_1, \dots, x_T$,
Algorithm~\ref{alg:0} has a mistake bound of
\[M_T(\cA^*, x_{1:T}) \leq
\frac{|S^*|(1-p)(1 - (1-p)^T)}{p} \leq \frac{|S^*|}{p}.\]
\label{thm:1}\label{thm:mistake}
\end{theorem}

\begin{proof}
Let $S_1, \dots, S_T$ be the sets such that on round $i$, Algorithm~\ref{alg:0} computes the function $f_{S_i}$.
Consider any element $e\in S^*$. Let $N_T(e)$ be the number of times $e \not\in S_i$ but $e\in S^*(x_i)$ for some $i \in [T]$. In other words, $N_T(e) = \left|\left\{i : e \not\in S_i, e \in S^*(x_i)\right\}\right|$.
Every time the algorithm makes a mistake, the current set $S_i$ must not contain some $e \in S^*(x_i)$ (otherwise, $S_i \supseteq S^*(x_i)$, so the algorithm would not have made a mistake by Assumption~\ref{assumption}). This means that every time the algorithm makes a mistake, $N_T(e)$ is incremented by 1 for at least one $e \in S^* = \cup_{i = 1}^TS^*(x_i)$. Therefore,
\begin{equation}M_T(\cA^*, x_{1:T}) \leq \sum_{e\in S^*}\E[N_T(e)],\label{eq:lb}\end{equation}
where the expectation is over the random choices of 
Algorithm~\ref{alg:0}.

For any element $e\in S^*$, let $i_1, \dots, i_t$ be the iterations where for all $\ell \in [t]$, $e \in S^*\left(x_{i_\ell}\right)$. By definition, $N_T(e)$ will only be incremented on some subset of these rounds. Suppose $N_T(e)$ is incremented by 1 on round $i_r$. It must be that $e \not\in S_{i_r}$, which means $S_{i_r} \not= \calU$, and thus $S_{i_r} = \bar{S}_{i_r}$. Since $e \not\in \bar{S}_{i_r}$, it must be that $e \not\in \bar{S}_{i_\ell}$ for $\ell \leq r$ since $\bar{S}_{i_r} \supseteq \bar{S}_{i_{r-1}} \supseteq \cdots \supseteq \bar{S}_{i_1}$. Therefore, in each round $i_{\ell}$ with $\ell < r$, Algorithm~\ref{alg:0} must not have computed $S^*\left(x_{i_{\ell}}\right)$, because otherwise $e$ would have been added to the set $\bar{S}_{i_{\ell}+1}$. We can bound the probability of these bad events as \[\pr\left[S_{i_r} = \bar{S}_{i_r} \text{ and Algorithm~\ref{alg:0} does not compute  }S^*\left(x_{i_{\ell}}\right) \text{ for }\ell < r\right] = \prod_{\ell = 1}^r\left(1 - p_{i_\ell}\right) \leq (1-p)^r.\]
As a result,
\begin{equation}\E[N_T(e)] \leq \sum_{r = 1}^t (1-p)^r \leq \sum_{r=1}^{T}(1-p)^r = \frac{(1-p)(1 - (1-p)^T)}{p}.\label{eq:ub}\end{equation} The theorem statement follows by combining Equations~\eqref{eq:lb} and \eqref{eq:ub}.
\end{proof}

\begin{cor}Algorithm~\ref{alg:0} with $p_i = \frac{1}{\sqrt{i}}$ has a mistake bound of
 $M_T(\cA^*, x_{1:T}) \leq |S^*|\sqrt{T}$.
  \label{cor:1}\label{cor:mistake}
\end{cor}

In the following theorem, we prove that the mistake bound in Theorem~\ref{thm:1} is nearly tight. In particular, we show that for any $k \in \{1, \dots, T\}$ there exists a random sequence of inputs $x_1, \dots, x_T$ such that $\E\left[|S^*|\right] \approx k$ and $M_T(\cA^*, x_{1:T}) = \frac{k(1-p)(1 - (1-p/k)^T)}{p}.$ This nearly matches the upper bound from Theorem~\ref{thm:1} of \[\frac{|S^*|(1-p)(1 - (1-p)^T)}{p}.\] The full proof is in Appendix~\ref{app:expansive}. In Section \ref{sec:lower}, we show that in fact, $\cA^*$ achieves a near optimal tradeoff between runtime and pruned subset size over all possible pruning-based repeated algorithms.  

 \begin{restatable}{theorem}{tight}
For any $p\in(0, 1]$, any time horizon $T$, and any $k \in \{1, \dots, T\}$, there
is a random sequence of inputs to
Algorithm~\ref{alg:0} such that \[\E[\left|S^*\right|] = k\left(1 - \left(1 - \frac{1}{k}\right)^T\right)\] and its expected mistake bound with $p_i = p$ for all $i \in [T]$ is
\[\E[M_T(\calA^*, \vec{x}_{1:T})] = \frac{k(1-p)(1 - (1-p/k)^T)}{p}.\]
The expectation is over the sequence of inputs.
\end{restatable}

\begin{proof}[Proof sketch]
We base this construction on shortest-path routing. There is a fixed graph $G = (V,E)$ where $V$ consists of two vertices labeled $s$ and $t$ and $E$ consists of $k$ edges labeled $\{1, \dots, k\}$, each of which connects $s$ and $t$. The set $X = \left\{\vec{x}^{(1)}, \dots, \vec{x}^{(k)}\right\}$ consists of $k$ possible edge weightings. Under the edge weights $\vec{x}^{(i)}$, the edge $i$ has a weight of 0 and all other edges $j\not= i$ have a weight of 1. We prove the theorem by choosing an input at each round uniformly at random from $X$.
\end{proof}

In Theorem~\ref{thm:1}, we bounded the expected number of mistakes
Algorithm~\ref{alg:0} makes.
Next, we bound $\E\left[\frac{1}{T}\sum |S_i|\right]$, where $S_i$ is the set such that Algorithm~\ref{alg:0} outputs $f_{S_i}(x_i)$ in round $i$ (so either $S_i=\bar{S}_i$ or $S_i=\calU$, depending on the algorithm's random choice). In our applications, minimizing $\E\left[\frac{1}{T}\sum |S_i|\right]$ means minimizing the search space size, which roughly amounts to minimizing the average expected runtime of Algorithm~\ref{alg:0}.

\begin{theorem}
  For any inputs $x_1, \dots, x_T$, let $S_1, \dots, S_T$ be the sets such that on round $i$, Algorithm~\ref{alg:0} computes the function $f_{S_i}$. Then \[\E\left[\frac{1}{T}\sum_{i = 1}^T|S_i|\right] \leq |S^*| + \frac{1}{T}\sum_{i = 1}^T p_i(|\calU| - |S^*|),\]
  where the randomness is over the coin tosses of
  Algorithm~\ref{alg:0}. 
  \label{thm:rr}
\end{theorem}

\begin{proof} 
We know that for all $i$, $S_i = \calU$ with probability $p_i$ and $S_i = \bar{S}_i$ with probability $1-p_i$. Therefore, \[\E\left[\sum_{i = 1}^T |S_i|\right] = \sum_{i = 1}^T \E[|S_i|]= \sum_{i = 1}^T p_i |\calU| + (1-p_i) \E[|\bar{S}_i|] \leq \sum_{i = 1}^T p_i |\calU| + (1-p_i) |S^*|,\] where the final inequality holds because $\bar{S}_i \subseteq S^*$ for all $i \in [T]$.
\end{proof}

If we set $p_i = 1/\sqrt{i}$ for all $i$, we have the following corollary, since $\sum_{i=1}^T p_i \le 2 \sqrt{T}$.
\begin{cor}Given a set of inputs $x_1, \dots, x_T$, let $S_1,...,S_T$ be the sets such that on round $i$, Algorithm~\ref{alg:0} computes the function $f_{S_i}$. If $p_i = \frac{1}{\sqrt{i}}$ for all $i\in[T]$, then
\[\E\left [\frac{1}{T}\sum_{i=1}^T |S_i|\right ] \le |S^*| + \frac{2(|\mathcal{U}|-|S^*|)}{\sqrt{T}},\]
where the expectation is over the random choices of
Algorithm~\ref{alg:0}.
  \label{cor:pruning_size}
\end{cor}

\subsection{Instantiations of Algorithm~\ref{alg:0}}\label{sec:instantiations}
We now revisit and discuss instantiations of Algorithm~\ref{alg:0} for the three applications outlined in Section~\ref{sec:model}: shortest-path routing, linear programming, and string search. For each problem, we describe how one might compute the sets $S^*(x_i)$ for all $i \in [T].$

\paragraph{Shortest-path routing.} 
In  this setting, the algorithm computes the true shortest path $f(x)$ using, say, Dijkstra's shortest-path algorithm, and the set $S^*(x)$ is simply the union of edges in that path.
Since $S^* = \cup_{i = 1}^T S^*(x_i)$,  the mistake bound of $\left|S^*\right|\sqrt{T}$ given by Corollary~\ref{cor:mistake} is particularly strong when the shortest path does not vary much from day to day. Corollary~\ref{cor:pruning_size} guarantees that the average edge set size run through Dijkstra's algorithm is at most $|S^*| + \frac{2(|E|-|S^*|)}{\sqrt{T}}$. Since the worst-case running time of Dijkstra's algorithm on a graph $G' = (V', E')$ is $\tilde{O}(|V'| + |E'|)$, minimizing the average edge set size is a good proxy for minimizing runtime.

\paragraph{Linear programming.}
In the context of linear programming, computing the set $S^*(\vec{x}_i)$ is equivalent to computing $f(\vec{x})$ and returning the set of tight constraints.
Since $S^* = \cup_{i = 1}^T S^*(\vec{x}_i)$, the mistake bound of $|S^*| \sqrt{T}$  given 
by Corollary~\ref{cor:mistake} is strongest when the same constraints are tight across most timesteps. Corollary~\ref{cor:pruning_size} guarantees that the average constraint set size considered in each round is at most $|S^*| + \frac{2(m-|S^*|)}{\sqrt{T}}$, where $m$ is the total number of constraints. Since many well-known solvers take time polynomial in $|S_i|$ to compute $f_{S_i}$, minimizing $\E\left [\sum |S_i|\right ]$ is a close proxy for minimizing runtime.

\paragraph{String search.}
In this setting, the set $S^*(q_i,p_i)$ consists of the smallest index $j$ such that $p_i = \left(q_i^{(j)}, q_i^{(j+1)}, \dots, q_i^{(j+m-1)}\right)$, which we denote $j_i^*$. This means that computing $S^*(q_i,p_i)$ is equivalent to computing $f(q_i, p_i)$. 
The mistake bound of $|S^*| \sqrt{T} = \left | \cup_{i = 1}^T \left\{j_i^*\right\}\right | \sqrt{T}$ given by Corollary~\ref{cor:mistake} is particularly strong when the matching indices are similar across string pairs. Corollary~\ref{cor:pruning_size} guarantees that the average size of the searched index set in each round is at most $|S^*| + \frac{2(n-|S^*|)}{\sqrt{T}}$. Since the expected average running time of our algorithm using the na\"ive string-matching algorithm to compute $f_{S_i}$ is $\E\left [\frac{m}{T}\sum_{i=1}^T |S_i|\right ]$, minimizing $\E\left [\frac{1}{T}\sum_{i=1}^T |S_i|\right ]$ amounts to minimizing runtime.

%% file: lowerbound.tex
We now prove a lower bound on the tradeoff between runtime and the number of mistakes made by any repeated algorithm. We analyze a shortest path problem with two nodes $s$ and $t$ connected by $m+1$ parallel edges ($E=[m+1]$). Thus, all paths are single edges.
For any $m\geq 1$ and $T>1$, consider the following distribution $\mu_{m, T}$ over $T$-tuples of edge weights in $\mathbb{R}^{(m+1)\times T}$: \begin{itemize}
\item The weight on edge $m+1$ is always 1/2.
\item An edge $e \in [m]$ and integer $r \in [2T]$ are chosen uniformly at random. The weight on edge $e$ is 1 on periods preceding $r$ and 0 from periods $r$ or later.
\item The weight on every other edge in $[m]\setminus \{e\}$ is 1 on every period.
\end{itemize}
Note that because $r \in [2T]$, with probability 1/2, $r>T$ and edge $m+1$ 
will be the unique shortest path ($S^*=\{m+1\}$) for all $T$ periods. Otherwise, $S^*=\{e,m+1\}$. We say that an algorithm $\calA$ inspects an edge $e$ on period $i$ if it examines the memory location associated with that edge. 

\begin{restatable}{theorem}{LB}\label{th:lowerbound}
Fixing $m \geq 1$ and any even integer $T > 1$, any repeated algorithm $\calA$ must satisfy:
$$\E[m+\text{number of inspections made by } \calA] \cdot \E[1 + \text{number of mistakes made by } \calA] \geq mT/8,$$
where the expectation is over the random edge weights and the randomness of the algorithm.
\end{restatable}
The total number of inspections the algorithm $\calA$ makes is clearly a lower bound on its total runtime, so Theorem~\ref{th:lowerbound} demonstrates a tradeoff between runtime and accuracy. This theorem is tight up to constant factors, as can be seen by the trivial algorithm that inspects every edge until it encounters a 0 on some edge $e$ and then outputs that edge henceforth, which makes no mistakes and runs in expected time $\Theta(mT)$. Conversely, the algorithm that always outputs edge $m+1$ does not make any inspections and makes $\Theta(T)$ expected mistakes.

Finally our algorithm, when run with fixed $p$, achieves at most $\Theta(1/p)$ expected mistakes and makes an expected $\Theta(pmT)$ number of inspections and hence matches this tradeoff, up to constant factors. The proof of Theorem \ref{th:lowerbound} is deferred until Appendix \ref{ap:lowerbound}.

%% file: experiments.tex
In this section, we present experimental results for shortest-path routing and linear programming.

\paragraph{Shortest-path routing.}
We test Algorithm~\ref{alg:0}'s performance on real-world street maps, which we access via Python's OSMnx package~\citep{Boeing17:OSMnx}. Each street is an edge in the graph and each intersection is a node. The edge's weight is the street's distance. We run our algorithm for 30 rounds (i.e., $T = 30$) with $p_i = 1/\sqrt{i}$ for all $i \in [T]$. On each round, we randomly perturb each edge's weight via the following procedure. Let $G = (V, E)$ be the original graph we access via Python's OSMnx package. Let $\vec{x} \in \R^{|E|}$ be a vector representing all edges' weights. On the $i^{th}$ round, we select a vector $\vec{r}_i \in \R^{|E|}$ such that each component is drawn i.i.d. from the normal distribution with a mean of 0 and a standard deviation of 1. We then define a new edge-weight vector $\vec{x}_i$ such that $x_i[j] = \ind_{\{x[j] + r_i[j] > 0\}}\left(x[j] + r_i[j]\right).$ In Appendix~\ref{app:experiments}, we experiment with alternative perturbation methods.

In Figures~\ref{fig:Pittsburgh_map} and \ref{fig:Pittsburgh_graph}, we illustrate our algorithm's performance in Pittsburgh. Figure~\ref{fig:Pittsburgh_map} illustrates the nodes explored by our algorithm over $T = 30$ rounds. The goal is to get from the upper to the lower star. The nodes colored grey are the nodes Dijkstra's algorithm would have visited if we had run Dijkstra's algorithm on all $T$ rounds. The nodes colored black are the nodes in the pruned subgraph after the $T$ rounds. Figure~\ref{fig:Pittsburgh_graph} illustrates the results of running our algorithm a total of 5000 times ($T = 30$ rounds each run). The top (orange) line shows the number of nodes Dijkstra's algorithm explored averaged over all 5000 runs. The bottom (blue) line shows the average number of nodes our algorithm explored. Our algorithm returned the incorrect path on a $0.068$ fraction of the $5000\cdot T = 150,000$ rounds. In Appendix~\ref{app:experiments}, we show a plot of the average pruned set size as a function of the number of rounds.

\begin{figure}
\centering
\begin{subfigure}{.3\textwidth}
\includegraphics[width=\textwidth]{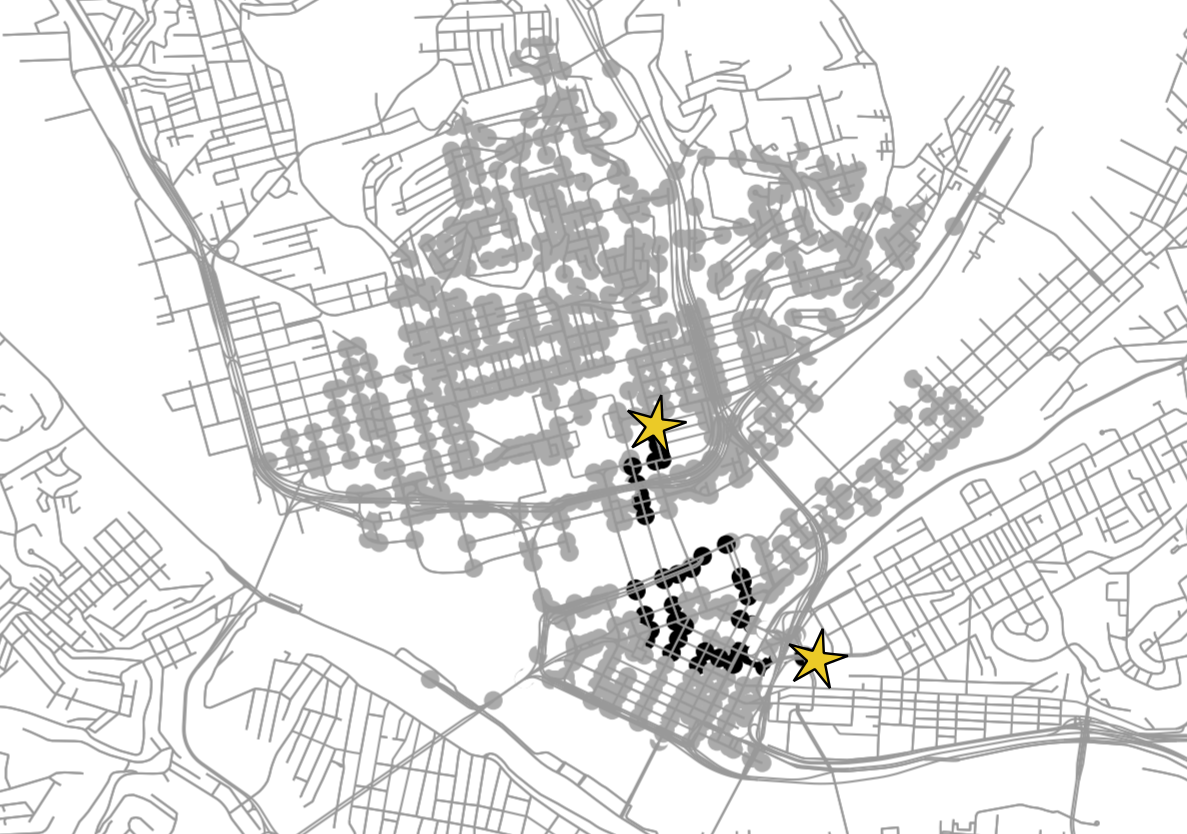}\centering
\caption{Grey nodes: the nodes visited by Dijkstra's algorithm. Black nodes: the nodes in our algorithm's pruned subgraph.}\label{fig:Pittsburgh_map}
\end{subfigure}\qquad
\begin{subfigure}{.3\textwidth}
\includegraphics[width=\textwidth]{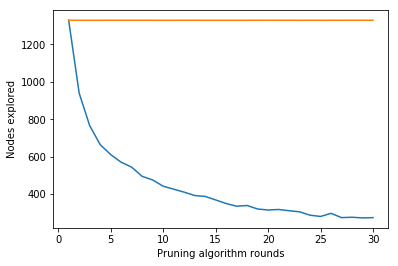}\centering
\caption{Top line: average number of nodes Dijkstra's algorithm explores. Bottom line: average number of nodes Algorithm~\ref{alg:0} explores.}\label{fig:Pittsburgh_graph}
\end{subfigure}\qquad
\begin{subfigure}{.3\textwidth}
\includegraphics[width=\textwidth]{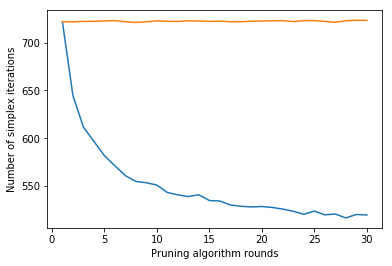}\centering
\caption{Top line: average number of simplex iterations the simplex algorithm makes. Bottom line: average number of simplex iterations Algorithm~\ref{alg:0} makes.}\label{fig:LP}
\end{subfigure}
\caption{Empirical evaluation of Algorithm~\ref{alg:0} applied to shortest-path routing in Pittsburgh (Figures~\ref{fig:Pittsburgh_map} and \ref{fig:Pittsburgh_graph}) and linear programming (Figure~\ref{fig:LP}).}
\label{fig:experiments}
\end{figure}

\paragraph{Linear programming.} We generate linear programming instances representing the linear relaxation of the combinatorial auction winner determination problem. See Appendix~\ref{app:experiments} for the specific form of this linear relaxation. We use the Combinatorial
Auction Test Suite (CATS)~\citep{Leyton00:Toward} to generate these instances. This test suite is meant to generate instances that are realistic and economically well-motivated. We use the CATS generator to create an initial instance with an objective function defined by a vector $\vec{x}$ and constraints defined by a matrix $A$ and a vector $\vec{b}$. On each round, we perturb the objective vector as we describe in Appendix~\ref{app:LP_experiments}.

From the CATS ``Arbitrary'' generator, we create an instance with 204 bids and 538 goods which has 204 variables and 946 constraints. We run Algorithm~\ref{alg:0} for 30 rounds ($T = 30$) with $p_i = 1/\sqrt{i}$ for all $i \in [T]$, and we repeat this 5000 times. In Figure~\ref{fig:LP}, the top (orange) line shows the number of simplex iterations the full simplex algorithm makes averaged over all 5000 runs. The bottom (blue) line shows the number of simplex iterations our algorithm makes averaged over all 5000 runs. We solve the linear program on each round using the SciPy default linear programming solver~\citep{Jones01:SciPy}, which implements the simplex algorithm~\citep{Dantzig16:Linear}.  Our algorithm returned the incorrect solution on a $0.018$ fraction of the $5000\cdot T = 150,000$ rounds. In Appendix~\ref{app:experiments}, we show a plot of the average pruned set size as a function of the number of rounds.

%% file: conclusion.tex
We propose an algorithm for quickly solving a series of related problems. Our algorithm learns irrelevant regions of the solution space that may be pruned across instances. With high probability, our algorithm makes few mistakes, and it may prune large swaths of the search space.
For problems where the solution can be checked much more quickly than found (such as linear programming), one can also check each solution and re-run the worst-case algorithm on the few errors to ensure zero mistakes. In other cases, there is a tradeoff between the mistake probability and runtime. 

%% file: appendix_expansive.tex
\tight*

\begin{proof}
We base our construction on the shortest path problem. There is a fixed graph $G = (V,E)$ where $V = \{s,t\}$ consists of two vertices labeled $s$ and $t$ and $E$ consists of $k$ edges labeled $\{1, \dots, k\}$, each of which connects $s$ and $t$. The set $X = \left\{\vec{x}^{(1)}, \dots, \vec{x}^{(k)}\right\} \subset \{0,1\}^k$ consists of $k$ possible edge weightings, where \[x^{(i)}[j] = \begin{cases}0 &\text{if } i = j\\
1 &\text{if }i \not=j.\end{cases}\] In other words, under the edge weights $\vec{x}_i$, the edge $i$ has a weight of 0 and all other edges $j\not= i$ have a weight of 1. The shortest path from $s$ to $t$ under the edge weights $\vec{x}^{(i)}$ consists of the edge $i$ and has a total weight of 0. Therefore, $f\left(\vec{x}^{(i)}\right) = S^*\left(\vec{x}^{(i)}\right) = \{i\}$. Given any non-empty subset of edges $S \subseteq E$, $f_S\left(\vec{x}^{(i)}\right) = \{i\}$ if $i \in S$, and otherwise breaks ties according to a fixed but arbitrary tie-breaking rule.

To construct the random sequence of inputs from the theorem statement, in each round $i$ we choose the input $\vec{x}_i$ uniformly at random from the set $X$. Therefore, letting $S^* = \cup_{i = 1}^TS^*(\vec{x}_i)$, $\E[\left|S^*\right|] = k\left(1 - \left(1 - \frac{1}{k}\right)^T\right)$, because when throwing $T$ balls uniformly at random into $k$ bins, the expected number of empty bins is $k\left(1 - \frac{1}{k}\right)^T$.

We now prove that $$\E[M_T(\calA^*, \vec{x}_{1:T})] = \frac{k(1-p)(1 - (1-p/k)^T)}{p},$$
where the expectation is over the sequence of inputs. To this end, let $S_1, \dots, S_T$ be the sets such that at round $i$, Algorithm~\ref{alg:0} outputs $f_{S_i}(\vec{x}_i)$.
\begin{align*}&\E[M_T(\calA^*, \vec{x}_{1:T})]\\
=\text{ }& \sum_{i = 1}^T \pr[\text{Mistake made at round }i]\\
=\text{ }& \sum_{i = 1}^T \sum_{j = 1}^k\pr[S_i = \bar{S}_i, \, \vec{x}^{(j)} = \vec{x}_i \text{, and } j\not\in \bar{S}_i]\\
=\text{ }& \sum_{i = 1}^T \sum_{j = 1}^k\pr[j\not\in \bar{S}_i \mid S_i = \bar{S}_i \text{ and } \vec{x}^{(j)} = \vec{x}_i]\cdot \pr[S_i = \bar{S}_i \text{ and } \vec{x}^{(j)}  = \vec{x}_i]\\
=\text{ }& \sum_{i = 1}^T \sum_{j = 1}^k\pr[j\not\in \bar{S}_i \mid S_i = \bar{S}_i \text{ and } \vec{x}^{(j)}  = \vec{x}_i]\cdot \pr[S_i = \bar{S}_i]\cdot \pr[\vec{x}^{(j)}  = \vec{x}_i]\\
=\text{ }& \frac{(1-p)}{k}\sum_{i = 1}^T \sum_{j = 1}^k\pr[j\not\in \bar{S}_i \mid S_i = \bar{S}_i \text{ and } \vec{x}^{(j)}  = \vec{x}_i].
\end{align*}

Analyzing a single summand,
\begin{align*}
&\pr[j\not\in \bar{S}_i \mid S_i = \bar{S}_i \text{ and } \vec{x}^{(j)} = \vec{x}_i]\\
=\text{ }&\sum_{t = 0}^{i-1} \Pr[\vec{x}_j \text{ is the input on exactly } t \text{ rounds before round } i \text{ and on those rounds, } S_\ell = \bar{S}_\ell]\\
=\text{ }&\sum_{t = 0}^{i-1} {i-1 \choose t}\left(\frac{1-p}{k}\right)^t \left(1 - \frac{1}{k}\right)^{i - 1 - t}\\
=\text{ }& \left(1 - \frac{p}{k}\right)^{i-1}.
\end{align*}
Therefore, \begin{align*}\E[M_T(\calA^*, \vec{x}_{1:T})] &= \frac{1-p}{k}\sum_{i = 1}^{T} \sum_{j = 1}^k\left(1 - \frac{p}{k}\right)^{i-1}\\
&= (1-p)\sum_{i =1}^{T}\left(1 - \frac{p}{k}\right)^{i-1}\\
&= \frac{k(1-p)\left(1 - \left(1 - p/k\right)^{T}\right)}{p},\end{align*} as claimed.
\end{proof}

%% file: appendix_lowerbound.tex
\LB*

\begin{proof}
First, consider a deterministic algorithm $\calA$. We say that $\calA$ \emph{inspects edge $e$ on period $i$} if it examines the memory location for edge $e$'s weight on period $i$. Let $n$ be the total number of edges that would be inspected during the first $T$ periods if $r>T$, which is well defined since when $r>T$, the weights are all fixed and hence the choices of the deterministic algorithm on the first $T$ periods are fixed. In fact, since $r>T$ with probability $1/2$ the expected number of inspections is at least $n/2$. 

We next observe that before $\calA$ has inspected an edge whose weight is 0, WLOG we may assume that $\calA$ chooses edge $m+1$ -- this minimizes its expected number of mistakes since edge $m+1$ has probability at least 1/2 of being the best conditional on any number of weight-1 inspections. (Once it inspects a 0 on $e$, of course runtime and mistakes are minimized by simply choosing $e$ henceforth without any further inspections or calculations.)

Let $B\subseteq [m]\times[T]$ be the set of edges and times that are not inspected by $\calA$ (excluding edge $m+1$). We now bound the expected number of mistakes in terms of $b=|B|=mT-n$. We no longer assume $r>T$, but we can still use $n$ as it is well defined. For each $(e', r') \in B$, there will be a mistake on $r'$ if $e=e'$ and $r\leq r'$ and $(e,r), (e,r+1), \ldots, (e,r') \in B$. Hence, we divide $B$ into a collection $\mathcal{I}$ of maximal  consecutive intervals, $I_{aij}=\{(a,i), (a,i+1), \ldots, (a,j)\} \subseteq B$, where either $i=1$ or $(a,i-1) \not\in B$ and $j=T$ or $(a,j+1)\not\in B$. Let $b=|B|$ denote the total length of all such intervals, which is $b=mT-n$. For any such interval $I\in \mathcal{I}$, there is a probability of $|I|/(2mT)$ that $(e,r)\in I$, because there is a 1/2 probability that $r \leq T$ and, conditional on this, $(e,r)$ is uniform from $[m]\times [T]$. Moreover, conditional on $(e,r) \in I$, the expected number of mistakes is $(1+|I|)/2$ because this is the expected length of the part of the interval that is at or after $r$. Hence, the expected number of mistakes is,
\begin{equation}\label{eq:convex}
\E[\text{total mistakes}] \geq \sum_{I \in \mathcal{I}} \frac{(1+|I|)|I|}{4mT}=\frac{b}{4mT} + \sum_{I \in \mathcal{I}} \frac{|I|^2}{4mT}.\end{equation}
The number of intervals is at most $N = m + n$ because: (a) with no inspections, there are $m$ intervals, (b) each additional inspection can create at most 1 interval by splitting an interval into two, and (c) there are $n$ edge-period inspections. By the convexity of the function $f(x)=x^2$ and (\ref{eq:convex}), the lower-bound on expected number of mistakes from above is at least,
$$\E_{(e,r)}[\text{total mistakes}] \geq \frac{b}{4mT} + N \frac{(b/N)^2}{4mT} = \frac{b(N+b)}{4mTN} = \frac{(mT-n)(m + mT)}{4(m + n)mT} \geq \frac{mT}{4(m+n)}-1.$$
using $b=mT-n$ and simple arithmetic.

This gives the lower bound on,
$$\E_{(e,r)}[m+\text{number of inspections}] \cdot \E_{(e,r)}[1 + \text{total mistakes}]\geq (m+n/2)\frac{mT}{4(m+n)} \geq \frac{mT}{8},$$
as required for deterministic $\calA$. To complete the proof, we need to show that the above holds, in expectation, for randomized $\calA$ as well. For a given randomized algorithm $\calA_z$ with random bits $z$, we can consider two quantities,
$$V_z = \E_{(e,r)}[m+\text{number of inspections of } \calA_z] \text{ and } W_z = \E_{(e,r)}[1 + \text{total mistakes of } \calA_z].$$
Now, we know that for all specific $z$, $\calA_z$ is a deterministic algorithm and hence $V_z W_z \geq mT/8$ for all $z$. Finally note that the set $S=\{(V,W)\in \mathbb{R}_+^2~|~ VW \geq mT/8\}$ is a convex set and since $(V_z, W_z) \in S$ for all $z$,  $(\E_z[V_z],\E_z[W_z]) \in S$ by convexity.
\end{proof}

%% file: appendix_experiments.tex
\begin{figure}
\centering
\begin{subfigure}{.47\textwidth}
\includegraphics[width=\textwidth]{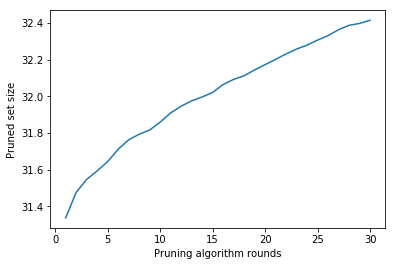}\centering
\caption{Average pruned set size for shortest-path routing.}\label{fig:SP_prunedsetsize}
\end{subfigure}\qquad
\begin{subfigure}{.47\textwidth}
\includegraphics[width=\textwidth]{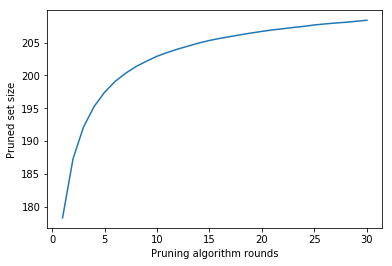}\centering
\caption{Average pruned set size for linear programming.\newline}\label{fig:LP_prunedsetsize}
\end{subfigure}
\caption{Average size of the pruned set $\bar{S}_i$ in Algorithm~\ref{alg:0}.}
\label{fig:pruning}
\end{figure}

Figure~\ref{fig:pruning} illustrates the average size of the pruned set $\bar{S}_i$ in Algorithm~\ref{alg:0}, which we ran a total of 5000 times, with $T = 30$ rounds each run. Figure~\ref{fig:SP_prunedsetsize} corresponds to shortest-path routing and Figure~\ref{fig:LP_prunedsetsize} corresponds to linear programming, with the same setup as described in Section~\ref{sec:experiments}.

\subsection{Shortest-path routing}\label{app:SP_experiments}
\begin{figure}
\centering
\begin{subfigure}{.47\textwidth}
\includegraphics[width=\textwidth]{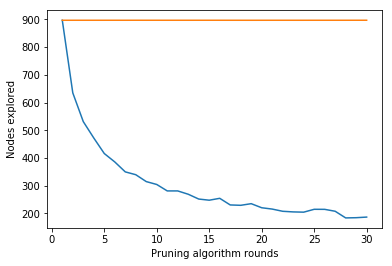}\centering
\caption{Gaussian perturbation with a standard deviation of $1/2$. Top line: average number of nodes Dijkstra's algorithm explores. Bottom line: average number of nodes Algorithm~\ref{alg:0} explores.}\label{fig:SP_iterations_gaussian_0_5}
\end{subfigure}\qquad
\begin{subfigure}{.47\textwidth}
\includegraphics[width=\textwidth]{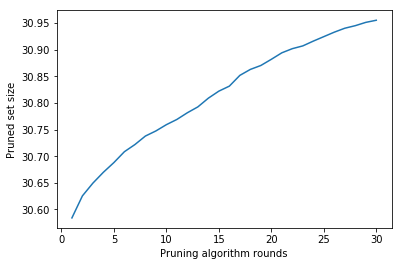}\centering
\caption{Gaussian perturbation with a standard deviation of $1/2$.  Average pruned set size for shortest path routing.\newline}\label{fig:SP_prunedsetsize_gaussian_0_5}
\end{subfigure}
\begin{subfigure}{.47\textwidth}
\includegraphics[width=\textwidth]{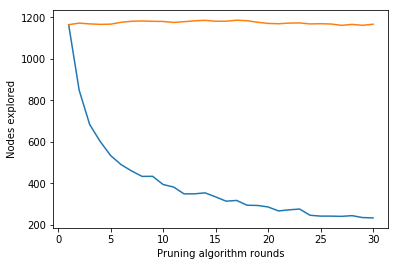}\centering
\caption{Uniform perturbation with support $[-1/2, 1/2]$. Top line: average number of nodes Dijkstra's algorithm explores. Bottom line: average number of nodes Algorithm~\ref{alg:0} explores.}\label{fig:SP_iterations_uniform_0_5}
\end{subfigure}\qquad
\begin{subfigure}{.47\textwidth}
\includegraphics[width=\textwidth]{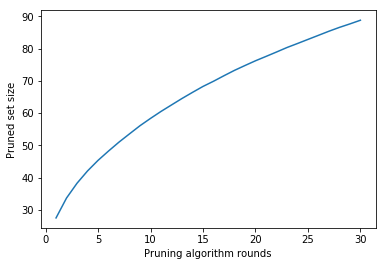}\centering
\caption{Uniform perturbation with support $[-1/2, 1/2]$.  Average pruned set size for shortest path routing.\newline\newline}\label{fig:SP_prunedsetsize_uniform_0_5}
\end{subfigure}
\begin{subfigure}{.47\textwidth}
\includegraphics[width=\textwidth]{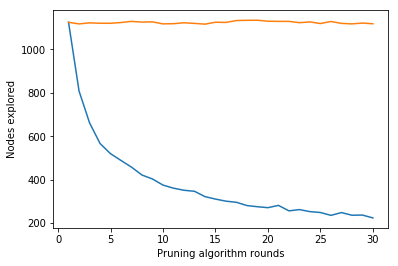}\centering
\caption{Uniform perturbation with support $[-1, 1]$. Top line: average number of nodes Dijkstra's algorithm explores. Bottom line: average number of nodes Algorithm~\ref{alg:0} explores.}\label{fig:SP_iterations_uniform_1}
\end{subfigure}\qquad
\begin{subfigure}{.47\textwidth}
\includegraphics[width=\textwidth]{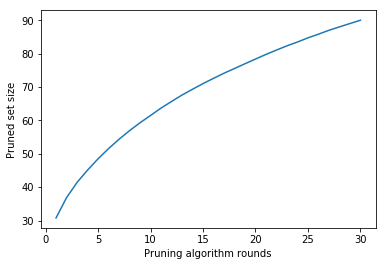}\centering
\caption{Uniform perturbation with support $[-1/2, 1/2]$.  Average pruned set size for shortest path routing.\newline\newline}\label{fig:SP_prunedsetsize_uniform_1}
\end{subfigure}
\caption{Shortest path routing experiments using varying perturbation methods.}
\label{fig:perturbation}
\end{figure}
In Figure~\ref{fig:perturbation}, we present several additional experiments using varying perturbation methods. Figures~\ref{fig:SP_iterations_gaussian_0_5} and \ref{fig:SP_prunedsetsize_gaussian_0_5} have the same experimental setup as in the main body except we employ a Gaussian distribution with smaller variance. We run our algorithm for thirty rounds (i.e., $T = 30$) with $p_i = 1/\sqrt{i}$ for all $i \in [T]$. On each round, we randomly perturb each edge's weight via the following procedure. Let $G = (V, E)$ be the original graph we access via Python's OSMnx package. Let $\vec{x} \in \R^{|E|}$ be a vector representing all edges' weights. On the $i^{th}$ round, we select a vector $\vec{r}_i \in \R^{|E|}$ such that each component is drawn i.i.d. from the normal distribution with a mean of 0 and a standard deviation of $1/2$. We then define a new edge-weight vector $\vec{x}_i$ such that $x_i[j] = \ind_{\{x[j] + r_i[j] > 0\}}\left(x[j] + r_i[j]\right).$ Our algorithm returned the incorrect path on a $0.034$ fraction of the $5000\cdot T = 150,000$ rounds.

In the remaining panels of Figure~\ref{fig:perturbation}, we employ the uniform distribution rather than the Gaussian distribution. In Figures~\ref{fig:SP_iterations_uniform_0_5} and \ref{fig:SP_prunedsetsize_uniform_0_5}, 
we run our algorithm for thirty rounds (i.e., $T = 30$) with $p_i = 1/\sqrt{i}$ for all $i \in [T]$. On each round, we randomly perturb each edge's weight via the following procedure. Let $G = (V, E)$ be the original graph we access via Python's OSMnx package. Let $\vec{x} \in \R^{|E|}$ be a vector representing all edges' weights. Let $w_i = \min\{x[i], 1/2\}$. On the $i^{th}$ round, we select a vector $\vec{r}_i \in \R^{|E|}$ such that each component is drawn from the uniform distribution over $[-w_i, w_i]$. We then define a new edge-weight vector $\vec{x}_i$ such that $x_i[j] = x[j] + r_i[j].$ Our algorithm returned the incorrect path on a $0.003$ fraction of the $5000\cdot T = 150,000$ rounds.

In Figures~\ref{fig:SP_iterations_uniform_1} and \ref{fig:SP_prunedsetsize_uniform_1}, we use the same procedure except $w_i = \min\{x[i], 1\}$. In that setting, our algorithm returned the incorrect path on a $0.001$ fraction of the $5000\cdot T = 150,000$ rounds.

\subsection{Linear programming}\label{app:LP_experiments}
We use the CATS generator to create an initial instance with an objective function defined by a vector $\vec{x}$ and constraints defined by a matrix $A$ and a vector $\vec{b}$. On the $i^{th}$ round, we select a new objective vector $\vec{x}_i$ such that each component $x_i[j]$ is drawn independently from the normal distribution with a mean of $x[j]$ and a standard deviation of 1. We run our algorithm for twenty rounds (i.e., $T = 30$) with $p_i = \frac{1}{\sqrt{i}}$ for all $i \in [T]$.

\paragraph{Winner determination.}
Suppose there is a set $\{1, \dots, m\}$ of items for sale and a set $\{1, \dots, n\}$ of buyers. In a combinatorial auction, each buyer $i$ submits bids $v_i(b)$ for any number of bundles $b \subseteq \{1, \dots, m\}$. The goal of the winner determination problem is to allocate the goods among the bidders so as to maximize \emph{social welfare}, which is the sum of the buyers' values for the bundles they are allocated. We can model this problem as a integer program by assigning a binary variable $x_{i,b}$ for every buyer $i$ and every bundle $b$ they submit a bid $v_i(b)$ on. The variable $x_{i,b}$ is equal to 1 if and only if buyer $i$ receives the bundle $b$. Let $B_i$ be the set of all bundles $b$ that buyer $i$ submits a bid on. An allocation is feasible if it allocates no item more than once ($\sum_{i = 1}^n \sum_{b \in B_i, j \ni b} x_{i,b} \leq 1$ for all $j \in \{1, \dots, m\}$) and if each bidder receives at most one bundle ($\sum_{b \in B_i} x_{i,b} \leq 1$ for all $i \in \{1, \dots, n\}$).
Therefore, the integer program is:
\[
\begin{array}{lll}
\textnormal{maximize} & \sum_{i = 1}^n \sum_{b \in B_i} v_i(b) x_{i,b} &\\
\textnormal{s.t.} & \sum_{i = 1}^n \sum_{b \in B_i, j \ni b} x_{i,b} \leq 1 &\forall j \in [m]\\
&\sum_{b \in B_i} x_{i,b} \leq 1 &\forall i \in [n]\\
& x_{i,b} \in \{0,1\} &\forall i \in [n], b \in B_i.
\end{array}
\]

To transform this integer program into a linear program, we simply require that $x_{i,b} \in [0,1]$ for all $i \in [n]$ and all $b \in B_i$.